\documentclass{article}
\usepackage{spconf,amsmath,graphicx}
\usepackage{amssymb}
\usepackage{amsthm}
\usepackage{multirow}
\usepackage{stfloats}
\usepackage{algorithm}
\usepackage{algorithmic}
\newtheorem{definition}{Definition}
\newtheorem{theorem}{Theorem}

\title{UNSUPERVISED DISTANCE METRIC LEARNING FOR ANOMALY DETECTION OVER MULTIVARIATE TIME SERIES}
\name{Hanyang Yuan$^{1}$, Qinglin Cai$^{2}$, Keting Yin$^{1}$}
\address{$^{1}$ School of Software Technology, Zhejiang University, Hangzhou, China \\
$^{2}$Donghai Laboratory, Zhoushan, China}
\begin{document}
\ninept
\maketitle
\begin{abstract}
% Anomaly detection for multivariate time series (MTS) has various applications. Due to the robustness and precision, dynamic time warping (DTW) has achieved success for univariate time series anomalies. However, extending DTW to MTS presents challenges owing to complex dependencies among dimensions. While existing works have explored distance metric learning methods for supervised tasks, unsupervised methods for DTW remain scarce. To address this, we propose FCM-wDTW, an unsupervised distance metric learning method that improves fuzzy C-means clustering by introducing locally weighted DTW. FCM-wDTW encodes raw data into latent space, revealing normal dimension relationships through cluster centers. By efficient optimization, the optimal latent space and distance measurement is learned, enabling anomaly identification via data reconstruction. Experiments with 11 different types of benchmarks demonstrate our method’s competitive accuracy and efficiency.
Distance-based time series anomaly detection methods are prevalent due to their relative non-parametric nature and interpretability. However, the commonly used Euclidean distance is sensitive to noise. While existing works have explored dynamic time warping (DTW) for its robustness, they only support supervised tasks over multivariate time series (MTS), leaving a scarcity of unsupervised methods. In this work, we propose FCM-\emph{w}DTW, an unsupervised distance metric learning method for anomaly detection over MTS, which encodes raw data into latent space and reveals normal dimension relationships through cluster centers. FCM-\emph{w}DTW introduces locally weighted DTW into fuzzy C-means clustering and learns the optimal latent space efficiently, enabling anomaly identification via data reconstruction. Experiments with 11 different types of benchmarks demonstrate our method’s competitive accuracy and efficiency.
\end{abstract}
\begin{keywords}
Anomaly Detection, Multivariate Time Series, Unsupervised Learning, Dynamic Time Warping
\end{keywords}

\section{Introduction}
\label{sec: intro}
The anomaly detection for multivariate time series presents a significant challenge for the intricate correlation and dependency among dimensions. It has been studied in various application fields, e.g., financial fraud detection\cite{huang2018codetect}, network intrusion detection \cite{kumar2018anomaly}, and industrial control fault diagnosis \cite{garg2021evaluation}. Anomaly is classically defined as \emph{an observation that deviates significantly from other observations, raising suspicions that it was generated by a different mechanism} \cite{hawkins1980identification}. Generally, MTS-generating systems remain in stable normal status, characterized by consistent relationships among dimensions. The goal of multivariate anomaly detection is to identify abnormal relationships that disrupt normal status.

Most existing anomaly detection methods are semi-supervised or unsupervised \cite{schmidl2022anomaly, blazquez2021review}, as anomaly labels are typically rare and costly.
While the former methods train models on the normal samples and mark the points or subsequences going against the model as anomalous, the latter separate anomalies from the normal part by frequency, shape, or distribution. The state-of-the-art deep learning methods for multivariate anomaly detection are mostly semi-supervised \cite{schmidl2022anomaly}. Comparatively, they need prior knowledge of data for training, limiting their availability in practical application. 
% While former methods train models on the normal samples and mark the points or subsequences going against the model as anomalous, the latter separates anomalies from the normal part by frequency, shape, or distribution. Comparatively, semi-supervised methods need prior knowledge of data for training, limiting their availability in practical application. 
Furthermore, in terms of recent survey works \cite{schmidl2022anomaly, audibert2022deep,kim2022towards}, including  \cite{schmidl2022anomaly}
% Furthermore, recent research \cite{schmidl2022anomaly}
that estimates 71 anomaly detection methods on 976 time series datasets, deep learning methods have yet to prove real outperformance, while they have relatively high complexity and poor interpretability. In contrast, the conventional statistical analysis and machine learning methods are simpler, lighter, and easier to interpret.

Among existing works, distance-based anomaly detection methods are particularly prevalent due to their non-parametric property and strong interpretability \cite{schmidl2022anomaly}. They are mostly based on nearest neighbors or clustering \cite{breunig2000lof, ramaswamy2000efficient, li2021clustering}, where the choice of distance metric plays a crucial part. However, the commonly used Euclidean distance is noise-sensitive and may negatively affect performance. To address this issue, DTW has been successfully applied to anomaly detection on univariate time series \cite{benkabou2018unsupervised}. Nevertheless, extending DTW to MTS remains a challenge due to the complex dependency among dimensions, which affects the local pointwise measure and the whole precision further. Previous attempts involve taking advantage of the PCA similarity factor \cite{banko2012correlation} or parameterized distance metric \cite{cai2015piecewise, mei2015learning, shen2017novel, li2021clustering} as the pointwise measure of DTW. However, they only support supervised tasks, and no sound unsupervised distance metric learning method has been proposed so far.

To bridge this gap, we propose an unsupervised distance metric learning method for DTW in multivariate anomaly detection. To be specific, we enhance fuzzy C-means clustering (FCM) by introducing a locally weighted DTW (\emph{w}DTW) as the distance metric. FCM-\emph{w}DTW encodes raw data into latent space, where cluster centers represent the normal relationships among dimensions of MTS. Then the optimal latent space and \emph{w}DTW are learned by an efficient closed-form optimization algorithm. Finally, anomalies are identified by reconstructing data from the latent space. For the samples far away from any clusters, their reconstruction errors are large and indicate anomalous. Our contributions are summarized as follows:
\begin{itemize}
% \vspace{-0.1in}
\setlength{\topsep}{0pt}
\setlength{\itemsep}{0pt}
\setlength{\parsep}{0pt}
\setlength{\parskip}{0pt}
\item	We improve FCM clustering by introducing \emph{w}DTW, which enhances the clustering accuracy over MTS. And build an efficient optimization algorithm for it.
\item 	Based on the learned latent space and \emph{w}DTW, we propose a multivariate anomaly detection method with data reconstruction and anomaly scoring computation.
\item 	Extensive experiments with 11 unsupervised multivariate anomaly detection benchmarks demonstrate the favorable performance of FCM-\emph{w}DTW.
\end{itemize}
\section{Problem statement}
\label{sec: problem}
MTS can be seen as a sequence of sampling vectors over correlated variables. Specifically, a MTS $T$ consists of a sequence of $n$ observations, i.e., $T=\{\boldsymbol{t}_1, \boldsymbol{t}_2, \ldots, \boldsymbol{t}_n\}$, where $\boldsymbol{t}_i \in \mathbb{R}^d$ and $d>1$. The anomaly detection over MTS is formally defined as follows.

\begin{definition}[Multivariate Anomaly Detection] 
\label{def1: problem}
% \vspace{-0.1in}
Given a MTS $T$, multivariate anomaly detection aims at computing an anomaly score $s_i$ for each observation $\boldsymbol{t}_i$, such that higher is $s_i$, more anomalous is $\boldsymbol{t}_i$.
\vspace{-0.1in}
\end{definition}
Note that in Definition \ref{def1: problem}, We make no assumptions about whether anomalies are point or subsequence anomalies. If the anomaly scores of continuous observations are high, these observations can be detected as a subsequence anomaly.
% \vspace{-0.1in}
\section{Method}
\label{sec: method}
In this section, we first introduce the objective and optimization algorithm of FCM-\emph{w}DTW, providing a comprehensive analysis of the algorithm. Then, we describe the anomaly detection method based on the learned latent space and \emph{w}DTW.
\subsection{Objective of FCM-\emph{w}DTW}
% We first describe the proposed unsupervised distance metric learning method FCM-\emph{w}DTW. 
The computation of DTW over MTS contains two layers: the pointwise measure on sampling vectors and the dynamic programming (DP) algorithm over the pointwise cost matrix (PCM). To regulate MTS dimensions, we parameterize DTW by employing the weighted Euclidean distance (WED) as the pointwise measure. Given two $w$-dimensional MTS $X=\{\boldsymbol{x}_1, \ldots, \boldsymbol{x}_i, \ldots, \boldsymbol{x}_m\}$ and $Y$ $=\{\boldsymbol{y}_1, \ldots, \boldsymbol{y}_j, \ldots, \boldsymbol{y}_n\}$, where $\boldsymbol{x}_i=[x_{i1}, x_{i2}, \ldots, x_{iw}]$ and $\boldsymbol{y}_j=[y_{j1},y_{j2}, \ldots, y_{jw}]$ are sampling vectors, and the optimal warping path (OWP) $\boldsymbol{p}=\{p_k \in M \mid p_k=(i, j), 1 \leq i \leq m, 1 \leq j \leq n, 1 \leq k $ $\leq l\}$, where $M \in \mathbb{R}^{m \times n}$ is the PCM and $p_1=(1,1), p_l=(m, n)$, $p_{k+1}-p_k \in\{(1,0),(0,1),(1,1)\}$, by introducing WED, we can get \emph{w}DTW formulated as
\begin{equation}
\mathrm{\emph{w}DTW}(X, Y)=\sum_p \mathrm{wed} (\boldsymbol{x}_i, \boldsymbol{y}_j)=\sum_p \sum_{d=1}^w \lambda_d^q\left(x_{i d}-y_{j d}\right)^2
\end{equation}
where $\mathrm{wed}(\boldsymbol{x}_i, \boldsymbol{y}_j)$ denotes the WED between $\boldsymbol{x}_i$ and $\boldsymbol{y}_j, \lambda_d\in[0,1]$ is the weight of the $d$-th dimension that satisfies $\sum_{d=1}^w \lambda_d=1$, and $q>1$ is a hyper-parameter.

To learn an adaptive \emph{w}DTW for the unsupervised tasks, we introduce it into FCM as the kernel distance metric. In terms of (1), the objective function of FCM-\emph{w}DTW can be formulated as:
\begin{align}
\label{equ: obj}
J(U, V, \Lambda) &=\min \sum_{i=1}^c \sum_{j=1}^n u_{i j}^m \cdot \mathrm{\emph{w}DTW}\left(v_i, x_j\right)  \notag\\
&=\min \sum_{i=1}^c \sum_{j=1}^n \sum_p \sum_{d=1}^w u_{i j}^m \lambda_d^q\left(v_{i d}-x_{j d}\right)^2 \notag\\
\mathrm{st}: 1)\sum_{i=1}^c &u_{i j}=1, \forall j=1,2, \ldots, n, 2)\sum_{d=1}^w \lambda_d=1
\end{align}
where $\Lambda=\left\{\lambda_d \mid d \in[1, w]\right\}$ denotes the set of weight parameters. $V =\left\{v_i \mid i \in[1, c]\right\}$ denotes cluster center set. $U =\left[u_{i j}\right]$ denotes the membership matrix, where $i \in[1, c], j \in[1, n], u_{i j} \in[0,1]$.

\subsection{Optimizing FCM-\emph{w}DTW} 
With the nondifferentiable \emph{w}DTW, it is difficult to optimize the objective function (2) directly. An alternating method is to solve the following four partial optimization sub-problems iteratively: 

% \noindent
\begin{itemize}
\setlength{\topsep}{0pt}
\setlength{\itemsep}{0pt}
\setlength{\parsep}{0pt}
\setlength{\parskip}{0pt}
\item \textbf{Problem 1} Keeping $U, V$, and $\Lambda$ fixed, update OWPs between $X$ and $V$;
\item \textbf{Problem 2} Keeping OWPs, $V$, and $\Lambda$ fixed, update $U$;
\item \textbf{Problem 3} Keeping OWPs, $U$, and $V$ fixed, update $\Lambda$;
\item \textbf{Problem 4} Keeping OWPs, $U$, and $\Lambda$ fixed, update $V$.
\end{itemize}

The four problems iteratively optimize the four factors determining the loss of FCM-\emph{w}DTW, which have explicit meaning. $V$ contains all cluster centers that reveal the normal patterns of MTS and construct the latent space. $U$ contains the encoding feature of samples that can be seen as the proportion of all normal patterns in composing a sample (where $u_{i j} \in[0,1]$ and $\sum_{j=1}^c u_{i j}=1$). $\Lambda$ models the correlation among dimensions of MTS while OWPs regulate the temporal relationship between samples and normal patterns. They make up the distance metric together in the latent space. Problem 1 can be solved by the DP algorithm and Problem $2 \sim 4$ can be solved by the Lagrange multiplier method, where (\ref{equ: obj}) is reformulated as:
\begin{align} 
\label{equ: lagrange}
L = \sum_{i=1}^c \sum_{j=1}^n & u_{i j}^m \cdot \mathrm{\emph{w}DTW}(v_i, x_j) \notag\\
&+\sum_{j=1}^n \delta_j \left(\sum_{i=1}^c u_{ij} - 1 \right)+\mu \left(\sum_{d=1}^w \lambda_d - 1 \right)
\end{align}

\begin{theorem}
\label{thm: 1}
As fixing the OWPs, $V$, and $\Lambda, U$ can be updated by:
\begin{equation} \label{equ: thm1}
u_{i j}=\frac{1}{\sum_{s=1}^c\left[\frac{\mathrm{\emph{w}DTW}\left(v_i, x_j\right)}{\mathrm{\emph{w}DTW}\left(v_s, x_j\right)}\right]^{\frac{1}{m-1}}}
\end{equation}
\end{theorem}
\begin{proof}  
By (\ref{equ: lagrange}), we have
\begin{equation}
\label{equ: proof1-1}
\frac{\partial L}{\partial u_{i j}}=m u_{i j}^{m-1} \cdot \mathrm{\emph{w}DTW} \left(v_i, x_j\right)+\delta_j
\end{equation}
As the OWPs, $V$, and $\Lambda$ are fixed, the \emph{w}DTW distance is constant. By setting $\frac{\partial L}{\partial u_{i j}}=0$, we have
\begin{equation}
\label{equ: proof1-2}
u_{i j}=\left(\frac{-\delta_j}{m}\right)^{\frac{1}{m-1}}\left[\frac{1}{\mathrm{\emph{w}DTW}\left(v_i, x_j\right)}\right]^{\frac{1}{m-1}}
\end{equation}
By substituting (\ref{equ: proof1-2}) into the first constraint of (\ref{equ: obj}), we have
\begin{equation}
\label{equ: proof1-3}
\left(\frac{-\delta_j}{m}\right)^{\frac{1}{m-1}}=\frac{1}{\sum_{i=1}^c\left[\frac{1}{\mathrm{\emph{w}DTW}\left(v_i, x_j\right)}\right]^{\frac{1}{m-1}}}
\end{equation}
By substituting (\ref{equ: proof1-3}) into (\ref{equ: proof1-2}), the proof is complete.
\end{proof}
\begin{theorem}
% \vspace{-0.2in}
\label{thm: 2}
As fixing the OWPs, $U$, and $V$, $\Lambda$ can be updated by
\begin{equation} \label{equ: thm2}
\lambda_d=\frac{1}{\sum_{s=1}^w\left(\frac{A_d}{A_s}\right)^{\frac{1}{q-1}}}
\end{equation}
where $A_d=\sum_i^c \sum_j^n \sum_p u_{i j}^m\left(v_{i d}-x_{j d}\right)^2$ denotes the intra-cluster distance on the $d$-th dimension. 
\end{theorem}
The proof of Theorem \ref{thm: 2} follows a similar procedure to that of Theorem \ref{thm: 1}, and will not be reiterated here. 
\begin{theorem}
\label{thm: 3}
As fixing the OWPs, $U$, and $\Lambda, V$ can be updated by
\begin{equation} \label{equ: thm3}
v_{i d r}=\frac{\sum_{j=1}^n u_{i j}^m \sum_{s=1}^a x_{j d s} \cdot I(r, s)}{\sum_{j=1}^n u_{i j}^m \sum_{s=1}^a I(r, s)}, r=1, \ldots, b
\end{equation}
where $a, b$ are the lengths of $v_i$ and $x_j$ respectively. $I(r, s)$ is an indicator function, $I(r, s)=1$ if $(r, s) \in \boldsymbol{p}$ otherwise $0$, where $\boldsymbol{p}$ is the OWP between $v_i$ and $x_j$.
\end{theorem}
\begin{proof}
By (3), we have
$$
\frac{\partial L}{\partial v_{i d r}}=\sum_{j=1}^n u_{i j}^m \lambda_d^q \sum_{s=1}^a 2\left(v_{i d r}-x_{j d s}\right) \cdot I(r, s)
$$
By setting $\frac{\partial L}{\partial v_{i d r}}=0$, we can get (\ref{equ: thm3}), and the proof is complete.
\end{proof}

Based on the solutions above, FCM-\emph{w}DTW can be optimized by iterating Problem 1$\sim$4, as shown in Algorithm \ref{alg:1}. Step 1 initializes cluster centers $V$ and the weights of dimensions $\Lambda$. One feasible initialization is to randomly choose $c$ samples from the dataset as the initial $V$. And the weights of dimensions $\Lambda$ can be randomly initialized within the range of [0, 1], while ensuring that they satisfy the condition $\sum_{d=1}^{w}\lambda_d=1$. Steps $2\sim9$ update the corresponding variables, respectively. Step 10 calculates the loss of the objective function and determines if the algorithm converges.

\begin{algorithm}[t] 
    \caption{FCM-\emph{w}DTW}
    \begin{algorithmic}[1]
        \REQUIRE MTS dataset $D = \{x_j|j\in[1, n]\}$, cluster number $c$, fuzzy coefficient $m$, exponent $q$, and loss threshold $\varepsilon$
        \ENSURE Membership matrix $U$, cluster centers $V$, weight parameters $\Lambda$
        \STATE \textbf{initialize} $V$, $\Lambda$
        \REPEAT
        \FOR{$x_i$, $x_j$}
            \STATE Calculate PCM $M_{ij}$
            \STATE Seek for $\mathrm{OWP}_{ij}$ by DP algorithm
        \ENDFOR
            \STATE Update $U$ by (\ref{equ: thm1})
            \STATE Update $\Lambda$ by (\ref{equ: thm2})
            \STATE Update $V$ by (\ref{equ: thm3})
        \UNTIL{$J(U, V, \Lambda) < \varepsilon$}
    \end{algorithmic}
    \label{alg:1}
% \setlength{\textfloatsep}{0pt}
% \vspace{-0.15in}
\end{algorithm}
\setlength{\textfloatsep}{0.1cm}
\setlength{\floatsep}{0.1cm}
\vspace{-0.15in}
\subsection{Algorithm Analysis}
\noindent \textbf{Complexity of FCM-\emph{w}DTW. }
Given the average length of samples $a$, the dimensionality $w(<<a)$, and the number of clusters $c(<<a)$, the computational complexity of seeking for the OWPs in Step 3$\sim$6 is $O\left(ncwa^2\right)$, updating the membership matrix in Step 7 is $O\left(n c^2\right)$, updating weight parameters in Step 8 is $O\left(ncwa+w^2\right)$, and updating cluster centers in Step 9 is $O(ncwa)$. Suppose the algorithm iterates $e$ times, the complexity of the whole algorithm is $O\left(e n c w a^2\right)$.

\noindent \textbf{Hyperparameters analysis. }
FCM-\emph{w}DTW has two hyperparameters, i.e., the fuzzy coefficient $m$ and the exponent $q$ of the weight coefficient. By (\ref{equ: thm1}), as $m \rightarrow 1$, the membership matrix $U$ tends to be sparse and the algorithm tends to the hard clustering, while as $m \rightarrow+\infty, U$ tends to be average and the algorithm tends to the uniform fuzzy clustering. Previous study \cite{pedrycz2008development} has proved through experiments that the optimal values of $m$ are typically lower than the commonly used value of 2.0. Thus, we set $m$ to be within the range of $(1, 2]$. On the other hand, by (\ref{equ: thm2}), a larger weight can strengthen the contribution of the MTS dimension with a smaller intra-cluster distance $A_d$, and vice versa. In terms of this principle, we can investigate the values of $q$ within different ranges. Firstly, as $q=0$, $\lambda_d^q \equiv 1$ and the WED becomes the original Euclidean distance, against our aim of discriminating the MTS dimensions. As $0<q<1$, the larger $A_d$, the larger $\lambda_d^q$, violating the principle above. As $q=1$, the weight coefficient of the dimension with the smallest $A_d$ is one and the others are zero, meaning only a single dimension plays a role in clustering and the information loss would influence the clustering accuracy seriously. As $q>1$ or $q<0$, the larger $A_d$, the smaller $\lambda_d^q$, satisfying the expected principle. Thus, $q$ should be within the range of $(-\infty, 0) \cup(1$, $+\infty)$.
\subsection{Anomaly Detection}
In FCM-\emph{w}DTW, data is projected into latent space constructed by cluster centers, which is composed of the normal patterns only in terms of the proportion of membership degree. Intuitively, if we reconstruct data from this space, the reconstructions are expected to be as close to the cluster centers as possible, and the samples with abnormal components will have large reconstruction errors. Thus, the anomaly score can be computed based on the difference between the sample and its reconstruction.

Based on the optimal cluster centers, partition matrix, and \emph{w}DTW, the reconstructed samples can be obtained by solving the objective function of FCM-\emph{w}DTW directly. Let $\ddot{x}_j$ denote the reconstruction of $x_j$, the objective function of clustering $\ddot{x}_j$ with FCM-\emph{w}DTW can be obtained as follows
% \vspace{-0.1in}
\begin{align}
J(U, V, \Lambda)  &=\min \sum_{i=1}^c \sum_{j=1}^n u_{i j}^m \cdot \mathrm{\emph{w}DTW} \left(v_i, \ddot{x}_j\right) 
\notag\\
&=\min \sum_{i=1}^c \sum_{j=1}^n \sum_{p} \sum_{d=1}^w u_{i j}^m \lambda_d^q\left(v_{i d}-\ddot{x}_{j d}\right)^2
\end{align}
By zeroing the gradient of $J$ with respect to $\ddot{x}_j$, we have
\begin{equation}
\ddot{x}_{j d s}=\frac{\sum_{i=1}^c u_{i j}^m \sum_{r=1}^b v_{i d r} \cdot I(r, s)}{\sum_{i=1}^c u_{i j}^m \sum_{r=1}^b I(r, s)}, s=1, \ldots, a
\end{equation}
Then the anomaly score of each sample can be computed as (\ref{equ: anomaly score}), by the \emph{w}DTW distance between the raw data and its reconstruction.
\begin{equation}
\label{equ: anomaly score}
s_j=\mathrm{\emph{w}DTW}\left(x_j, \ddot{x}_j\right)
\end{equation}

\section{EXPERIMENTS}
In this section, we first estimate the performance of FCM-\emph{w}DTW on multivariate anomaly detection, which is conducted on 4 datasets against 11 benchmarks. We then examine the runtime efficiency of the optimization algorithm for FCM-\emph{w}DTW.
\begin{table*}[t]
\centering
\setlength{\tabcolsep}{5pt}
\normalsize
\caption{Anomaly detection accuracy comparison. The optimal results are highlighted in bold.}
\begin{tabular}{l|l|ll|ll|ll|ll}
\hline 
\multirow{2}{*}{Method} & \multirow{2}{*}{Type} &CalIt2& & PCSO5 & & PCSO10 & &PCSO20& \\
\cline { 3 - 10 }
& &   ROC-AUC  &  PR-AUC  &  ROC-AUC &  PR-AUC & ROC-AUC &PR-AUC &  ROC-AUC &PR-AUC \\
\hline
LOF \cite{breunig2000lof}      &Distance      & 0.727 &0.119 & 0.445 & 0.008 &	0.440 &	0.009 &	0.447 &	0.009 \\
KNN	\cite{ramaswamy2000efficient}    &Distance	   & 0.883 &0.264 & 0.681 & 0.014 &	0.752 &	0.019 &	0.599 &	0.011 \\
CBLOF \cite{he2003discovering}	&Distance	   & 0.871 &0.256 & 0.271 & 0.006 &	0.355 &	0.007 &	0.225 &	0.006 \\
HBOS \cite{goldstein2012histogram}	&Distance	   & 0.873 &0.250 & 0.207 &	0.006 &	0.297 &	0.007 &	0.185 &	0.006 \\
COF	\cite{tang2002enhancing}   &Distance	   & 0.839 &0.161 & 0.463 &	0.009 &	0.432 &	0.008 &	0.471 &	0.009 \\
EIF	\cite{hariri2019extended}    &Trees	       & 0.885 &0.259 & 0.948 &	0.083 &	0.847 &	0.032 &	0.856 &	0.033 \\
IF-LOF \cite{cheng2019outlier}	&Trees	       & 0.795 &0.122 & 0.610 &	0.012 &	0.603 &	0.012 &	0.619 &	0.013 \\
iForest	\cite{liu2008isolation} &Trees	       & 0.881 &0.259 & 0.859 &	0.033 &	0.838 &	0.030 &	0.836 &	0.029 \\
COPOD \cite{li2020copod}	&Distribution  & 0.884 &0.269 & 0.809 &	0.024 &	0.465 &	0.009 &	0.808 &	0.024 \\
PCC \cite{shyu2003novel}   &Reconstruction& 0.757 &0.240 & 0.874 &	0.038 &	0.600 &	0.012 &	0.870 &	0.037 \\
Torsk \cite{heim2019adaptive}  &Forecasting   & 0.585 &0.054 & 0.909 &	0.100 &	—     & —	  & 0.919 &	0.065 \\ 
FCM-\emph{w}DTW&Reconstruction& \textbf{0.904} & \textbf{0.465} & \textbf{0.993} & \textbf{0.818} & \textbf{0.974} & \textbf{0.423} & \textbf{0.952} & \textbf{0.345} \\
\hline
\end{tabular}
\label{tab: compare}
% \vspace{-0.15in}
\end{table*}

\subsection{Setup}
\noindent \textbf{Environment. }
The configuration is Intel(R) Core(TM) i9-12900k CPU @3.2GHz, 32GB memory, Ubuntu 20.04 OS. The programming language is Python 3.8.

\noindent \textbf{Datasets. }
\begin{figure}[ht]
% \vspace{-0.15in}
\begin{minipage}{.49\linewidth}
  \centering
  \centerline{\includegraphics[width=4.5cm]{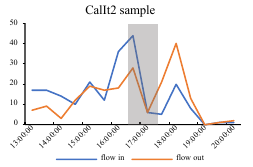}}
  % \centerline{(a) Results 3}\medskip
\end{minipage}
\hfill
\begin{minipage}{0.49\linewidth}
  \centering
  \centerline{\includegraphics[width=4.5cm]{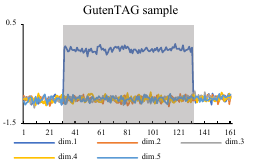}}
  % \centerline{(b) Result 4}\medskip
\end{minipage}
\caption{Calit2 and Gutentag samples.}
\label{fig: visual}
% \vspace{-0.15in}
\end{figure}
For estimating the anomaly detection accuracy, we consider four real-world datasets, namely CalIt \cite{asuncion2007uci} and PCSO5, PCSO10, and PCSO20 from GutenTAG \cite{schmidl2022anomaly}. CalIt2 comes from the data streams of people flowing in and out of the building of University of California at Irvine over 15 weeks, 48 time slices per day (half-hour count aggregates). The purpose is to predict the presence of an event such as a conference in the building that is reflected by unusually high people counts for that day/time period. GutenTAG is actually a time series anomaly generator that consists of single or multiple channels containing a base oscillation with a large variety of well-labeled anomalies at different positions. It generated time series of five base types (sine, ECG, random walk, cylinder bell funnel, and polynomial) with different lengths, variances, amplitudes, frequencies, and dimensions. The selection of injected anomalies covers nine different types. Fig. \ref{fig: visual} exhibits the samples from both data collections. In the CalIt2 sample, the relationship between two dimensions turns over with large amplitude in the shaded area, indicating a conference as an anomaly event. In the GutenTAG sample, the first dimension is injected with polynomial amplitude anomalies within the shaded area.

\noindent \textbf{Parameters. } 
To guarantee clustering robustness, the cluster centers of FCM-\emph{w}DTW are initialized with density peak clustering (DPC) \cite{rodriguez2014clustering}. By analysis ahead, the value ranges of the two hyper-parameters $m$ and $q$ are $(1.0, 2.0]$ and $[-10, 0)\cup(1, 10]$, the adjustment steps are 0.3 and 2 respectively. In addition, the sliding window size in anomaly detection is 16.

\noindent \textbf{Metrics. } 
We utilize two threshold-agnostic evaluation metrics. 1) AUC-ROC:  contrasts the TP rate with the FP rate (or Recall). It focuses on an algorithm’s sensitivity. 2) AUC-PR: contrasts the precision with the recall. It focuses on an algorithm’s preciseness.

\noindent \textbf{Baselines. } 
We adopt all the 11 unsupervised multivariate benchmarks published by the state-of-the-art survey work \cite{schmidl2022anomaly} (only except DBStream which has no results published), 
Table \ref{tab: compare} summarizes the name and type of each baseline.
% including five distance-based methods LOF \cite{breunig2000lof}, KNN \cite{ramaswamy2000efficient}, CBLOF \cite{he2003discovering}, HBOS 
% \cite{goldstein2012histogram}, COF \cite{tang2002enhancing}, three tree-based methods EIF \cite{hariri2019extended}, IF-LOF \cite{cheng2019outlier}, iForest \cite{liu2008isolation}, a distribution-based method, COPOD \cite{li2020copod}, a reconstruction method, PCC \cite{shyu2003novel}, and a forecasting deep learning method, Torsk \cite{heim2019adaptive}. 
To avoid implementation bias, we use the results from \cite{schmidl2022anomaly} directly. All baseline results are gained by optimizing the parameters globally for the best average AUC-ROC score.
\subsection{Accuracy}
We test the parameters of cluster number $c$ with values \{10, 20, 30, 40, 50\}. The optimal anomaly detection results of FCM-\emph{w}DTW and the benchmark results are reported in Table 1 respectively.
From the results, we note that for distance-based methods like LOF, KNN, and CBLOF, their performances are relatively worse. The reason behind this may be the sensitivity of Euclidean distance to noise.
Additionally, tree-based methods achieve a relatively high ROC-AUC on most datasets. However, they tend to have lower PR-AUC, indicating that they struggle to achieve a good trade-off between precision and recall.
In contrast, FCM-\emph{w}DTW not only achieves the best ROC-AUC but also shows a decent PR-AUC. This indicates that FCM-\emph{w}DTW is capable of achieving a favorable trade-off between the TP rate TPR and FP rate, as well as between precision and recall. Overall, FCM-\emph{w}DTW outperforms all other methods on four datasets in terms of both AUC-ROC and AUC-PR, signifying its effectiveness and robustness in accurately detecting anomalies.

\subsection{Runtime}
% To demonstrate the efficiency of the optimization algorithm of FCM-wDTW, we further compare 
We compare the real runtime of FCM-\emph{w}DTW against seven clustering benchmarks \cite{li2020fuzzy}, including CD, PDC, FCFW, FCMD-DTW, PAM-DTW, GAK-DBA and soft-DTW, on datasets CMUsubject16 and ECG \cite{baydogan2015multivariate}. The results are shown in Fig.\ref{fig: runtime}. Each method is repeated 100 times to calculate the average runtime. On CMUsubject16, except GAK-DBA, the runtime of all other methods is in the same order of magnitude. On ECG, except soft-DTW, the runtime of all other methods is in the same order of magnitude. In addition, the runtime of FCM-\emph{w}DTW is relatively low on ECG but high on CMUsubject16. We note that the iterations of FCM-\emph{w}DTW on both datasets are less than 10, but the sample lengths and the numbers of dimensions are different greatly. By (\ref{equ: thm2}), the extra cost of FCM-\emph{w}DTW on CMUsubject16 is caused by the procedure of initializing cluster centers and updating weight coefficients. Overall, although a more complex distance metric is introduced, the runtime of FCM-\emph{w}DTW remains comparable to that of other clustering methods.
\begin{figure}[t]
% \vspace{-0.15in}
\begin{minipage}[]{.49\linewidth}
  \centering
  \centerline{\includegraphics[width=4.5cm]{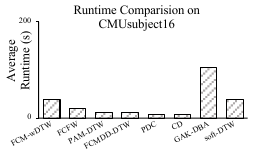}}
  % \centerline{(a) Results 3}\medskip
\end{minipage}
\hfill
\begin{minipage}[]{0.49\linewidth}
  \centering
  \centerline{\includegraphics[width=4.5cm]{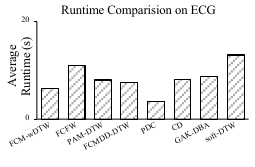}}
  % \centerline{(b) Result 4}\medskip
\end{minipage}
\vspace{-0.1in}
\caption{Runtime comparison of different clustering methods on CMUsubject16 and ECG.}
\label{fig: runtime}
% \vspace{-0.15in}
\end{figure}

\section{CONCLUSION}
In this work, we propose an unsupervised distance metric learning method based on FCM and \emph{w}DTW for multivariate anomaly detection. Our method solves the objective function in closed form for the reformulated optimization problem and builds an efficient optimization algorithm. The anomalies are identified by reconstructing data from the learned optimal latent space. Comprehensive experiments demonstrate the significant superiority of our methods. Future work will focus on examining its effectiveness on more practical problems, e.g., network performance monitoring, abnormal account detection, and attack behavior identification.

\vfill\pagebreak

\bibliographystyle{IEEEbib}
\bibliography{ref}

\end{document}